\documentclass{article}
\usepackage{arxiv}
\usepackage[utf8]{inputenc}
\usepackage{color}
\usepackage{amsmath,amssymb,amsthm}
\usepackage{mathtools,dsfont}

\usepackage{multirow, rotating}
\usepackage{xspace}

\usepackage
[vlined,ruled,titlenotnumbered]
{algorithm2e}
\usepackage{xcolor, soul}
\usepackage{colortbl}
\sethlcolor{yellow}
\newcommand{\cmmnt}[1]{}
\sethlcolor{lightgray}
\usepackage{xcolor}   
\usepackage{color} 
\usepackage{dblfloatfix}
\usepackage{xspace,dsfont}

\newcommand{\gsemo}{GSEMO\xspace}

\newcommand{\ie}{i.\,e.\xspace}

\newcommand{\prob}[1]{\mathrm{Pr}(#1)}

\usepackage{algpseudocode,booktabs,amsmath}

\newcommand{\semotD}{SEMO3D\xspace}
\newcommand{\pmax}{P_{\max}}

\title{3-Objective Pareto Optimization for Problems with Chance Constraints}
\author{Frank Neumann\\
Optimisation and Logistics\\
School of Computer and Mathematical Sciences\\
The University of Adelaide\\
Adelaide, Australia
\And
Carsten Witt\\
Algorithms, Logic and Graphs\\
DTU Compute\\ Technical University of Denmark\\
2800 Kgs. Lyngby Denmark
}
\newcommand{\ignore}[1]{}
\usepackage{graphicx}
\newtheorem{definition}{Definition}

\newtheorem{theorem}{Theorem}

\begin{document}
\maketitle

\begin{abstract}
Evolutionary multi-objective algorithms have successfully been used in the context of Pareto optimization where a given constraint is relaxed into an additional objective. In this paper, we explore the use of $3$-objective formulations for problems with chance constraints. Our formulation trades off the expected cost and variance of the stochastic component as well as the given deterministic constraint. We point out benefits that this 3-objective formulation has compared to a bi-objective one recently investigated for chance constraints with Normally distributed stochastic components. Our analysis shows that the $3$-objective formulation allows to compute all required trade-offs using $1$-bit flips only, when dealing with a deterministic cardinality constraint. Furthermore, we carry out experimental investigations for the chance constrained dominating set problem and show the benefit for this classical NP-hard problem.
\end{abstract}




\section{Introduction}
Evolutionary algorithms have been shown to be successful for a wide range of optimization problems and the development and application of evolutionary multi-objective algorithms~\cite{kdeb01,coello2013evolutionary} is one of the great success stories in the area of evolutionary computation.
This includes both solving classical problems with multiple objectives~(see e.g. \cite{DBLP:series/sci/RiponTK07,DBLP:journals/jors/LongZGP16,DBLP:conf/gecco/TsalavoutisVT19}) as well as using multi-objective models to solve single-objective problems by relaxing constraints into additional objectives~\cite{DBLP:journals/nc/NeumannW06,DBLP:journals/ec/FriedrichHHNW10,DBLP:journals/algorithmica/KratschN13} or adding helper objectives~\cite{DBLP:journals/jmma/Jensen04,DBLP:journals/tec/BrockhoffFHKNZ09,DBLP:journals/tec/LochtefeldC12}.
In the context of submodular optimization using multi-objective formulations that relax a given constraint into an additional objective has been shown to achieve best possible performance guarantees for a wide range of submodular problem while outperforming classical approaches based on greedy algorithms in practice~\cite{DBLP:conf/nips/QianYZ15,DBLP:conf/ijcai/QianSYT17,DBLP:conf/ppsn/NeumannN20}.

Tackling stochastic problems in terms of uncertainties can use objectives such as the expected cost (or value) and uncertainties such as variances or quantiles~\cite{DBLP:journals/tec/AsafuddoulaSR15,DBLP:conf/ppsn/SinghB22}.
Recently, chance constrained problems~\cite{Charnes} have gained increasing attention in the area of evolutionary computation~\cite{poojari,Zhang,DBLP:conf/gecco/XieN020,DBLP:conf/gecco/XieN0S21}. These problems involve stochastic components and constraints that should be met with a given probability $\alpha$.
Furthermore, formulations using chance constraints can be used to guarantee high quality objective function values with a high probability of $\alpha$ if the stochastic components influence the objective function and not (only) the constraint. 
It has been shown in \cite{DBLP:conf/ijcai/0001W22} that a bi-objective formulation taking into account the expected cost and variance of a solution is highly effective for minimizing the stochastic cost of a solution for every possible confidence level $\alpha$ when considering uniform or spanning tree constraints. 

We investigate the use of $3$-objective formulations instead of bi-objective formulations for the set up studied in \cite{DBLP:conf/ijcai/0001W22} by adding the constraint as an additional objective. Adding the constraint as an additional objective usually implies that the number of trade-offs according to the objectives functions grows significantly. On the other hand, the additional objective can enable a different way of searching for high quality solutions. We investigate in detail how $3$-objective formulations can be used for chance constrained problems.

Our first contribution is a theoretical runtime analysis which generalizes the results obtained in \cite{DBLP:conf/ijcai/0001W22} to the $3$-objective formulation. Here we show that the $3$-objective formulation allows to compute all possible trade-offs for independent Normally distributed weight for the whole set of possible uniform constraints. Investigating the problem furthermore, we show that in order to compute the whole set of trade-offs, only $1$-bit flips are required in the $3$-objective formulation which significantly improves upper the results given in \cite{DBLP:conf/ijcai/0001W22}.
Afterwards, we investigate the $3$-objective formulation and compare it to the bi-objective one through experimental investigations. We consider the chance constrained dominating set problem in the same set up as done in \cite{DBLP:conf/ijcai/0001W22}. Our results show that the $3$-objective formulation provides a clear advantage for graphs with up to $500$ nodes. 

The outline of the paper is as follows. In Section~\ref{sec:prelim}, we introduce the two chance constrained problem setups that we are investigating in this paper. In Section~\ref{sec:3D}, we introduce the $3$-objective formulation that is subject to our investigations. Sections~\ref{sec:2bits} and \ref{sec:1bits} provide a rigorous runtime analysis of our approach which shows that it efficiently computes a set up solutions that includes optimal solutions for a wide range of constrained settings of the two considered problems. We present our experimental results for the chance constrained dominating set problem in Section~\ref{sec:experiments} and finish with some conclusions.

\section{Preliminaries}
\label{sec:prelim}

Pareto optimization approaches are usually used to tackle constrained single-objective optimization problems by taking the constraint as an additional objective. 
Chance constrained problems involve constraints that are impacted by the expected (cost) value as well as its variance. In \cite{DBLP:conf/ijcai/0001W22}, a chance constrained problem has been considered which involves such stochastic components and has an additional deterministic constraint.
We motivate our multi-objective settings by these recent investigations. 

We consider the chance constrained problem investigated in \cite{DBLP:conf/ijcai/0001W22}. Given a set of $n$ items $I=\{1, \ldots, n\}$ with stochastic weights $w_i$, $1 \leq i \leq n$, we want to solve 
\begin{equation}
\min W  \text{~~~~subject to~~~~}  (\mathit{Pr}( w(x) \leq W) \geq \alpha) \wedge (|x|_1 \geq k)
\label{chance-problem}
\end{equation}
where $w(x) = \sum_{i=1}^n w_i x_i$, $x \in \{0,1\}^n$, and $\alpha \in \mathopen{[}1/2,1\mathclose{[}$.
The weights are independent and each $w_i$ is distributed according to a Normal distribution $N(\mu_i, \sigma_i^2)$, $1 \leq i \leq n$, where $\mu_i \geq 1$ and $\sigma_i\geq 1$, $1 \leq i \leq n$.
We denote by $\mu(x) = \sum_{i=1}^n \mu_i x_i$ the expected weight and by $v(x) = \sum_{i-1}^n \sigma_i^2 x_i$ the variance of the weight of solution $x$.

As stated in \cite{DBLP:conf/ijcai/0001W22}, the problem given in Equation~\ref{chance-problem} is equivalent to minimizing 
\begin{equation}
\label{eq:weight}
\hat{w}(x)=\mu(x) + K_{\alpha} \sqrt{v(x)},
\end{equation}
under the constraint that $|x|_1 \geq k$ holds. 
Here, $K_{\alpha}$ denotes the $\alpha$-fractional point of the standard Normal distribution.

 The uniform constraint $|x|_1\geq k$ requires that each feasible solution has to contain at least $k$ elements. As expected weights and variances are strictly positive, an optimal solution has exactly $k$ elements. Depending on the choice of $\alpha$, the difficulties lies in finding the right trade-off between the expected weight and variance among all solutions with exactly $k$ elements.

It has been shown that this problem given in Equation~\ref{chance-problem} can be solved by the following bi-objective formulation~\cite{DBLP:conf/ijcai/0001W22}.
The objective function is given as
$f_{2D}(x) = (\hat{\mu}(x), \hat{v}(x))$ 
where 
\begin{eqnarray*}
\hat{\mu}(x) = \begin{cases}
 \sum_{i=1}^n \mu_i x_i & |x|_1 \geq k\\
(k-|x|_1)\cdot (1+\sum_{i=1}^n \mu_i) & |x|_1<k
\end{cases}
\end{eqnarray*}

\begin{eqnarray*}
\hat{v}(x) = \begin{cases}
 \sum_{i=1}^n \sigma^2_i x_i & |x|_1 \geq k\\
(k-|x|_1)\cdot (1+\sum_{i=1}^n \sigma^2_i) & |x|_1<k
\end{cases}
\end{eqnarray*}

We say that a solution $x$ dominates a solution $y$ ($x \succeq y$) iff $\hat{\mu}(x) \leq \hat{\mu}(y) \wedge \hat{v}(x) \leq \hat{v}(y)$. Furthermore, a solution $x$ strongly dominates a solution $y$ ($x \succ y$) iff $x \succeq y$ and $f_{2D}(x) \not = f_{2D}(y)$. The setup can be generalized by using $c(x) \geq k$ for a constraint function $c(x)$ instead of $|x|_1=k$. In the experimental investigations carried out in \cite{DBLP:conf/ijcai/0001W22}, $c(x)$ is counting the number of dominated nodes in the dominating set problem in graphs with $n$ nodes, and $c(x)=n$ is required for a solution to be feasible.

The key idea of the result given in \cite{DBLP:conf/ijcai/0001W22} is to show that the algorithm computes the extremal points of the Pareto front of the given problem. Note that as the expected costs and variances are strictly positive, each Pareto optimal solution contains exactly $k$ elements when considering this bi-objective formulation.

We also consider the problem of maximizing a given deterministic objective $c(x)$ under a given chance constraint, i.e 
\begin{equation}
\max c(x)  \text{~~~~subject to~~~~}  \mathit{Pr}( w(x) \leq B) \geq \alpha.
\label{chance-problem2}
\end{equation}
with $w(x) = \sum_{i=1}^n w_i x_i$ where each $w_i$ is chosen independently of the other according to a Normal distribution $N(\mu_i, \sigma_i^2)$, and $B$ and $\alpha \in [1/2, 1[$ are a given weight bound and reliability probability.

Such a problem formulation includes for example the maximum coverage problem in graphs with so-called chance constraints~\cite{DBLP:conf/aaai/DoerrD0NS20,DBLP:conf/ppsn/NeumannN20}, where $c(x)$ denotes the nodes of covered by a given solution $x$ and the costs are stochastic. Furthermore, the chance constrained knapsack problem as investigated in \cite{DBLP:conf/gecco/XieN020,DBLP:conf/gecco/XieHAN019} fits into this problem formulation.

\section{3-Objective Pareto Optimization}
\label{sec:3D}

The now introduce the $3$-objective formulation of the problems given in Equation~\ref{chance-problem} and \ref{chance-problem2} and the algorithms that we study in this paper.

\subsection{3-Objective Formulation}

We investigate the 3-objective formulation given as
\[
f_{3D}(x)= ( \mu(x), v(x), c(x))
\]
where $c(x)$ is the constraint value of a given solution that should be maximized. In our theoretical study, we 
focus on the case $c(x) = |x|_1$, 
which 
turns the constraint $|x|_1 \geq k$ into the additional objective of maximizing the number of bits in the given bitstring. 

Similar to the bi-objective model we minimize the expected weight $\mu(x) = \sum_{i=1}^n \mu_i x_i$ and the variance $v(x) = \sum_{i=1}^n \sigma_i^2 x_i$ of the weight of solution $x$. Note that here we do not consider penalty terms for violating the constraint $|x|_1 \geq k$ as done in the bi-objective formulation.
We say that a solution $x$ dominates a solution $y$ ($x \succeq y$) iff $c(x) \geq c(y) \wedge \mu(x) \leq \mu(y) \wedge v(x) \leq v(y)$. Furthermore, a solution $x$ strongly dominates $y$ ($x \succ y$) iff $x \succeq y$ and $f_{3D}(x) \not = f_{3D}(y)$.

Generalizing the results given in \cite{DBLP:conf/ijcai/0001W22}, we show that our problem formulation solves the problem given in Equation~\ref{chance-problem} for every possible value of $k$ and $\alpha$, Furthermore, we use the 3-objective problem to compute, for any possible pair of $B$ and $\alpha$ values, a solution with the highest possible $c(x)$-value according to Equation~\ref{chance-problem2}.

Using the expected cost and variance as objectives for the problem given in Equation~\ref{chance-problem2}, allows here to explore the trade-offs with respect to the expected cost and variance for the different values of $B$ and $\alpha$ that lead to a maximum possible value of $c(x)$. We will show that the $3$-objective formulation is obtaining for any possible $B$ and $\alpha \in [1/2, 1]$ a feasible solution with the maximal value for $c(x)=|x|_1$ in expected pseudo-polynomial time. We first show this by adapting the proof given in \cite{DBLP:conf/ijcai/0001W22} to the $3$-objective setting. The proof makes use of specific $2$-bit flips that allow to compute all convex points of the Pareto front when constraining the number of elements to one particular constraint value $k$ and thereby solving the problem given in Equation~\ref{chance-problem} as well.

Afterwards, we improve our upper bound by showing that the $3$-objective formulation enables an additional search direction for evolutionary multi-objective algorithms which only relies on the use of $1$-bit flips. As specific $1$-bits occur more frequently than specific $2$-bit flips, we obtain an improved upper bound.

\subsection{Algorithms}
\begin{sloppypar}
For our investigations, we consider variants of the well-known Global Simple Evolutionary Multi-Objective Optimizer (GSEMO)~\cite{DBLP:journals/tec/LaumannsTZ04,1299908} given in Algorithm~\ref{alg:GSEMO}. The algorithm starts with one initial solution chosen uniformly at random and produces in each iteration a single offspring by standard bit mutations. \gsemo maintains at each point in the time for each non-dominated objective vector found so far one single solution. The variant of GSEMO called SEMO originally introduced in~\cite{DBLP:journals/tec/LaumannsTZ04} differs from GSEMO by flipping in each mutation step exactly one randomly chosen bit. 
\end{sloppypar}

We investigate the algorithms GSEMO2D and GSEMO3D which are using our bi-objective and $3$-objective problem formulation together with standard bit-mutations as outlined in Algorithm~\ref{alg:GSEMO}.
As the proofs for the bi-objective formulation carried out in \cite{DBLP:conf/ijcai/0001W22} rely on $1$- and $2$-bit flips, we consider the algorithm SEMO2D which with probability $1/2$ carries out a $1$-bit flip and otherwise carries out a $2$-bit flip in the mutation step. Similarly, as our investigations for the $3$-objective model show that it performs well with $1$-bit operations only, we consider the algorithm SEMO3D which flips in each mutation step one single bit. Note that SEMO3D is exactly the algorithm variant introduced in \cite{DBLP:journals/tec/LaumannsTZ04} although it has only been applied to bi-objective problems in that paper.

For our theoretical investigations, we measure time in terms of the number of fitness evaluations to achieve a desired goal. The expected number of fitness evaluations is also called the \emph{expected time} to achieve the given goal.

\begin{algorithm}[t]
 Choose $x \in \{0,1\}^n$ uniformly at random\;
 $P\leftarrow \{x\}$\;
\Repeat{$\mathit{stop}$}{
Choose $x\in P$ uniformly at random\;
Create $y$ by flipping each bit $x_{i}$ of $x$ with probability $\frac{1}{n}$\;
\If{$\nexists\, w \in P: w \prec y$} {
  $P \leftarrow (P \setminus \{z\in P \mid y \preceq z\}) \cup \{y\}$\;
  }
    }
\caption{GSEMO} \label{alg:GSEMO}
\end{algorithm}

\section{Analysis Based on $2$-Bit Flips}
\label{sec:2bits}
We investigate the problem given in Equation~\ref{chance-problem2} for the case where $c(x) = |x|_1$, and each $w_i$ is chosen according to the Normal distribution $N(\mu_i, \sigma^2_i)$ independently of the others. As done in \cite{DBLP:conf/ijcai/0001W22}, we assume $\mu_i \geq 1$ and $\sigma^2_i \geq 1$, $1 \leq i \leq n$, in the following. We claim that \gsemo computes an optimal solution for any combination of $B$ and $\alpha \in [1/2, 1[$ in expected pseudo-polynomial time for the Problem given in Equation~\ref{chance-problem2}. 

Inspired by the analysis of chance-constrained minimum spanning trees \cite{DBLP:journals/dam/IshiiSNN81}, we consider sets of Pareto optimal search points having exactly $k$, $0 \leq k \leq n$, elements that are
minimal with respect to 
\[
\hat{w}(x)=\mu(x) + K_{\alpha} \sqrt{v(x)}
\]
for each fixed $k$ and $\alpha$. 

To do this, we follow the ideas given in \cite{DBLP:conf/ijcai/0001W22}.
Our goal is to minimize
$f_{\lambda}(x) = \lambda \mu(x) + (1- \lambda) v(x)$. This can be done
by choosing iteratively $k$ minimal elements with respect to
$f_{\lambda}(e_i) = \lambda \mu_i + (1- \lambda) \sigma_i^2$, $0 < \lambda < 1$.
For $\lambda=0$ and $\lambda=1$, $f_{\lambda}$ is minimized by minimizing $f_0(x)=(v(x), \mu(x))$ and $f_1(x)=(\mu(x), v(x))$ with respect to the lexicographic order. 
Note that for each $\lambda \in [0,1]$, an optimal solution for $f_{\lambda}$ can be obtained by selecting the first $k$ items in increasing order with respect to $f_{\lambda}$.
In terms of notation, we use $f_{\lambda}$ for the evaluation of a search point $x$ as well as the evaluation of an element $e_i$ in the following.

We denote by $X^*_{k,\lambda} \subseteq \{0,1\}^n$ the set of minimal elements with respect to $f_{\lambda}$ having exactly $k$ elements. Note that all points in the sets $X^*_{k,\lambda}$, $0 \leq \lambda \leq 1$, are not strongly dominated in $\{0,1\}^n$ as the expected cost and variance strictly increase when adding any additional element.
Therefore, the sets $X^*_{k, \lambda}$, $0 \leq \lambda \leq 1$, $0 \leq k \leq n$,  constitute   Pareto optimal points. Note 
there may be other Pareto optimal points not included
in these sets.

\begin{definition}[Extreme point of set $X$]
\label{def:extremepoint}
For a given set $X\subseteq \{0,1\}^n$, we call $f(x) =(\mu(x),v(x))$ an extreme point of~$X$ if there is a $\lambda \in [0,1]$ such that $x \in X^*_{k, \lambda}$ and $v(x) = \max_{y \in X^*_{k, \lambda}} v(x)$.
\end{definition}

We denote by $P_{\max}$ the maximum population size that \gsemo encounters during the run of the algorithm, \ie, before reaching its goal of optimization.

Let $v_{\max} = \max_{1 \leq i \leq n} \sigma_i^2$ and $\mu_{\max}= \max_{1 \leq i \leq n} \mu_i$.
we assume that $v_{\max} \leq \mu_{\max}$ holds. Otherwise, the bounds in Theorem~\ref{thm:Pareto} and \ref{thm-2bit2} and can be tightened by replacing $v_{\max}$ by~$\mu_{\max}$.

Let $X^k = \{x \in \{0,1\}^n \mid |x|_1=k\}$ be the set of all solutions having exactly $k$ elements.
The following theorem shows that \gsemo computes for each $k$ and $\alpha$ an optimal solution for the problem given in Equation~\ref{chance-problem}, which has been investigated in \cite{DBLP:conf/ijcai/0001W22} in the 
context of the $2$-objective formulation given in Section~\ref{sec:prelim}.

\begin{theorem}
\label{thm:Pareto}
\gsemo computes a population $P$ which  contains for each $\alpha \in [1/2, 1[$ and $k \in \{0, \dots, n\}$ a solution 
\begin{equation}
x_{\alpha}^k = \arg \min_{x \in X^k} \left\{ \mu(x) + K_{\alpha} \sqrt{v(x)} \right\}
\label{eq:optimality-included} 
\end{equation} in expected time $O(\pmax n^4  (\log n + \log v_{\max}))$.
\end{theorem}

\begin{proof}
For $k \in \{0, n\}$, the search points $0^n$ and $1^n$ are the unique corresponding optimal solutions. 
As \gsemo minimizes $\mu(x)$, the search point $0^n$ is obtained in expected time $O(P_{\max}n \log n)$ using the bound on the population size together with standard fitness level argument from the analysis of the (1+1)~EA for OneMax \cite{Droste2002}. Similarly, as \gsemo maximizes $c(x)$, the search point $1^n$ is obtained in expected time $O(P_{\max}n \log n)$. These runtime bounds can be obtained by considering $1$-bit flips only.

We now consider the time to compute all Pareto optimal solutions $x^k_{\alpha}$ with exactly $k$ elements for $\alpha=1/2$ and $0 \leq k \leq n$  (the case of general $\alpha$ will be studied below). These are solution with the smallest expected weight as $\alpha=1/2$ implies $K_{\alpha}=0$.
Having computed all Pareto optimal solutions for $\alpha=1/2$ with at most $j$, $0\leq j<n-1$ elements, we pick the Pareto optimal solution $x^j_{\alpha}$ for $\alpha=1/2$ and $j$ elements.  Inserting the smallest element $i$ currently not $x^j_{\alpha}$ according to $(\mu_i, \sigma_i^2)$ in lexicographic order leads to a Pareto optimal solution $x^{j+1}_{\alpha}$ which has $j+1$ elements and is optimal for $\alpha=1/2$. The expected time for GSEMO to produce this solution is $O(P_{\max}n)$ as one can obtain it from $x^j_{\alpha}$ by flipping the single specific bit for element $i$. There are $n+1$ different values of $k$. Hence the expected time until an optimal solution $x^k_{\alpha}$ for $\alpha=1/2$ and all $k$, $0 \leq k \leq n$, has been obtained is $O(P_{\max}n^2)$.

We now analyze the time to get optimal solutions $x^k_{\alpha}$, $1 \leq k \leq n-1$, for any value of $\alpha>1/2$ after having obtained all optimal solutions for $\alpha=1/2$. 
Thereby, we make use of the arguments given in \cite{DBLP:conf/ijcai/0001W22} where it is shown that the set of Pareto optimal solutions when minimizing the bi-objective problem given as $g(x)=(\mu(x), v(x))$ under the constraint $|x|_1 \geq k$, contains all extreme points of the problem. As we have $\mu_i>0$, $\sigma^2_i>0$, $1 \leq i \leq n$, each Pareto optimal solution for minimizing $g(x)$ under the constraint $|x|_1 \geq k$ has exactly $k$ elements.
Let $X^k =\{x \in \{0,1\}^n \mid |x|_1 \geq k\}$ be the set of all solutions containing at least $k$ elements.
Furthermore let $Y^k = \{0,1\}^n \setminus X^k$ be the set of all search points having strictly less than $i$ elements. As the first objective is to maximize $p(x)$ no search point in $X^k$ is dominated by any search point in $Y^k$. Therefore creating a new search point in $Y^k$ does not lead to the removal of any element from $X^k$ if its currently contained in the population of \gsemo.

As done in \cite{DBLP:conf/ijcai/0001W22}, we define $\lambda_{i,j} = \frac{\sigma_j^2 - \sigma_i^2}{(\mu_i-\mu_j) +(\sigma^2_j - \sigma^2_i)}$ for the pair of items $i$ and $j$ where $\sigma^2_i < \sigma^2_j$ and $\mu_i > \mu_j$ holds, $1 \leq i < j \leq n$.
The set $\Lambda=\{\lambda_0, \lambda_1, \ldots, \lambda_{\ell}, \lambda_{\ell+1}\}$ where $\lambda_1, \ldots, \lambda_{\ell}$ are the values $\lambda_{i,j}$ in increasing order and 
$\lambda_0=0$ and $\lambda_{\ell+1}=1$. 
The set of Pareto optimal solutions with exactly $k$ elements contains all optimal solutions $f_{\lambda}$ and every $\lambda \in \Lambda$.

We will now analyze the time until \gsemo 
until \gsemo has computed a population which includes a solution  $x_\alpha^k$ according to \eqref{eq:optimality-included} for any choice of $\alpha \in [1/2,1\mathclose{[}$ and $k\in\{1,\dots,n-1\}$. To this end,
we build until the analysis leading to 
Theorem~4 in~\cite{DBLP:conf/ijcai/0001W22} but 
consider the progress of the algorithm for the $n-1$ different values for $|x|_1$ (exlucing the trivial 
values~$0$ and $n$)  
in parallel. Recall that the bi-objective 
formulation in \cite{DBLP:conf/ijcai/0001W22} considers only a fixed value (to be precise, a lower bound) for~$|x|_1$. The analysis in that paper proceeds by 
conducting a sequence of multiplicative drift analyses for the fixed~$k$; 
more precisely the expected time is bounded until a search point 
of minimal variance is obtained for $\lambda_0,\lambda_1,\dots,\lambda_\ell+1$,
using a multiplicative drift argument for each individual $\lambda$. 
Since there probability of choosing an individual that can be improved 
by a two-bit flip is at least~$1/(\pmax) $ and the 
probability of a two-bit flip is $\Omega(1/n^2)$, this results in a bound of 
$O(\pmax n^2 \ell (\log n + \log v_{\max}))$ in \cite{DBLP:conf/ijcai/0001W22}.

In our three-objective formulation, the probability of picking an 
individual whose first objective value equals~$k$ is at least $1/\pmax$. To progress of GSEMO towards its overall goal of including a solution $x_\alpha^k$ 
in the population  for each $\alpha\in [1/2,1\mathclose{[}$ and each number of one-bits  $k\in\{1,\dots,n-1\}$ is now estimated with $n-1$ parallel sequences of multiplicative drift processes in the following way. Considering a fixed~$k\in\{1,\dots,n-1\}$, sequence~$k$ corresponds to finding the optimal solutions for the uniform constraint of having at least $k$~one-bits. As explained 
above, 
this is guaranteed after completion of $\ell$ consecutive multiplicative drift processes indexed $1,\dots,\ell_k$. Each each point of time, each 
sequence $k$ is ``performing''
its $p_k$\text{th} multiplicative drift process, where $p_k$ ranges from $1$ to $\ell_k$ and increases by~$1$ when a process from the sequence 
has reached its target. 

 At each point of time,  exactly one $k$ is chosen and process $p_k$ from sequence~$k$ 
 decreases 
its state by a multiplicative factor of no larger than $1-\delta=1-1/(en^2)$ since there is always at least 
one two-bit flip available that decreases 
the state. For each multiplicative drift process 
within each of the $k$ sequences,  
 it
holds that it reaches the target after time $T^*=O(n^2 (\log n+\log w_{max}))$ with probability at least $1-2/n^3$, where 
we used the tail bounds for multiplicative drift. We also 
know from Lemma~2 in \cite{DBLP:conf/ijcai/0001W22} that $\ell_k\le n^2$ for every $k$.

For each $k\in\{1,\dots,n-1\}$, a step of sequence~$k$ happens with probability at least $1/\pmax$.
Within $2\pmax n^2 T^*$ steps, each sequence is chosen 
at least $n^2 T^*$ times with probability $1-e^{-\Omega(\pmax n^2 T^*)}
 = 1-e^{-\Omega(n^2)}$ according to Chernoff bounds. 
 Assuming this to happen, by a union bound over $n-1$ sequences and at most~$n^2$ processes 
per sequence, each process reaches its target within time~$T^*$ with 
probability at least $1-(n-1)n^2/(2n^3)
\ge 1/2$, implying that the every 
process from every sequence has reached its target within time $n^2 T^*$. 
The total failure probability is at most $1/2+e^{-\Omega(n^2)}=1/2+o(1)$. 

In case of a failure, we repeat the argumentation. The expected number of repetitions is at most $2+o(1)$, so 
that the total expected time until all processes have reached the target is 
$O((2+o(1))2\pmax n^2  T^*)=O(\pmax n^4 (\log n+\log w_{\max}))$. As explained 
above, this also bounds the time to include a
solution $x_\alpha^k$ 
in the population of \gsemo for 
all $\alpha\in [1/2,1\mathclose{[}$ and all $k\in \{1,\dots,n-1\}$. 
\end{proof}

The previous theorem bounds the time to 
include solutions of the type described in \eqref{eq:optimality-included}, which may look  abstract. The following theorem shows that these solutions allow us to find solutions for the chance-constrained problem in \eqref{chance-problem2} efficiently from the population of GSEMO
for different settings of confidence level and constraint.

\begin{theorem}
\label{thm-2bit2}
The expected time until \gsemo has computed a population which includes an optimal solution for the problem given in Equation~\ref{chance-problem2} with $c(x)=|x|_1$ for any possible choice of $B$ and $\alpha \in [1/2,1\mathclose{[}$ is $O(\pmax n^4 (\log n + \log v_{\max}))$,
\end{theorem}

\begin{proof}
We show that the population $P \supseteq \{x^k_{\alpha} \mid 0 \leq k \leq n, \alpha \in [1/2, \mathclose{[}$ given in Theorem~\ref{thm:Pareto} contains the optimal solutions for any choice of $\alpha \in [1/2,1\mathclose{[}$.
Let $x^*_{\alpha}$ be an optimal solution for a given value of $\alpha \in [1/2, 1[$.

For a given $\alpha$, the solution $x_{\alpha}^j$ with the maximum number of elements for which 
\[
\hat{w}(x)=\mu(x) + K_{\alpha} \sqrt{v(x)} \leq B
\]
holds satisfies $|x_{\alpha}^j|=|x^*_{\alpha}|$ as otherwise $x_{\alpha}^j$ would not be a solution with the maximal number of elements for which the constraint holds or $x^*_{\alpha}$ would not be optimal for $\alpha$.
This implies that $x_{\alpha}^j$ is an optimal solution for $\alpha$.
\end{proof}

\section{Improved Upper Bound Based on $1$-Bit Flips Only}
\label{sec:1bits}
The analysis from the 
previous section relied on specific $2$-bit flips that allow to produce the solutions for each value of $\alpha$ by swapping elements to produce new Pareto optimal solutions for a given number of $k$ elements.

We now show that $2$-bit flips are not necessary in the $3$-objective formulation and also improve the upper bound by considering only $1$-bit flips. We note that the upper bound 
is by an asymptotic factor  $\Omega(n^2)$ lower compared to Theorem~\ref{thm:Pareto} and 
includes the same $\pmax$.

\begin{theorem}
    The expected time until \semotD and \gsemo have computed a population which includes an optimal solution for the problems given in Equation~\ref{chance-problem} (for any choice of $k$ and $\alpha$) and Equation~\ref{chance-problem2} (with $c(x)=|x|_1$ for any choice of $B$ and $\alpha$) 
    is $O(\pmax n^2)$ and it is at most $2e\pmax n^2$ with probability $1 - e^{-\Omega(n)}$.
\end{theorem}

\begin{proof}
To prove the theorem, we show that the same set of Pareto optimal objective vectors can be computed by \gsemo as in the proof of Theorem~\ref{thm:Pareto} when considering $1$-bit flips only. 
Theorem~\ref{thm-2bit2} implies that then not only all optimal solutions with respect to Equation~\ref{chance-problem} but also with respect to  
Equation~\ref{chance-problem2} have been computed.

By a simple fitness-level argument,  
the expected time until the Pareto optimal search point $0^n$ has been included in the population is $O(\pmax n \log n)$. This search point will never be removed from the population as it is the unique search point with minimum expected cost and variance.

As done in \cite{DBLP:conf/ijcai/0001W22}, we define $\lambda_{i,j} = \frac{\sigma_j^2 - \sigma_i^2}{(\mu_i-\mu_j) +(\sigma^2_j - \sigma^2_i)}$ for the pair of items $i$ and $j$ where $\sigma^2_i < \sigma^2_j$ and $\mu_i > \mu_j$ holds, $1 \leq i < j \leq n$. 
The set $\Lambda=\{\lambda_0, \lambda_1, \ldots, \lambda_{\ell}, \lambda_{\ell+1}\}$ where $\lambda_1, \ldots, \lambda_{\ell}$ are the values $\lambda_{i,j}$ in increasing order and 
$\lambda_0=0$ and $\lambda_{\ell+1}=1$. 

Following \cite{DBLP:conf/ijcai/0001W22}, we define the function
\[
f_{\lambda}(x) = \lambda \mu(x) + (1-\lambda) v(x) \]
and also use it applied to elements $e_i$ of the given input, i.e. 
\[
 f_{\lambda}(e_i) = \lambda \mu_i + (1-\lambda) \sigma_i^2 .
\]

Note that for a given $\lambda$ the function $f_{\lambda}$ can be optimized by a greedy approach which iteratively selects a set of $k$ smallest elements according to $f_{\lambda}(e_i)$.
For any $\lambda \in [0,1\mathclose{[}$ an optimal solution for $f_{\lambda}$ with $k$ elements is Pareto optimal as there is no other solution with at least $k$ elements that improves the expected cost or variance without impairing the other. Hence, once obtained a solution with the resulting objective vector will remain in the population for the rest of the optimization process. 
Furthermore, the set of optimal solutions for different $\lambda$ values only change at the $\lambda$ values of the set $\Lambda$ as these $\lambda$ values constitute the weightening where the order of items according to $f_{\lambda}$ can switch~\cite{DBLP:journals/dam/IshiiSNN81,DBLP:conf/ijcai/0001W22}.

We consider a $\lambda_i \in \Lambda$ with $0 \leq i \leq \ell$ and similar to \cite{DBLP:journals/dam/IshiiSNN81} define $\lambda_i^* = (\lambda_i + \lambda_{i+1})/2$. The order of items according to the weightening of expected value and variance can only change at values $\lambda_i \in \Lambda$ and the resulting objective vectors are not necessarily unique for values $\lambda_i \in \Lambda$.  Choosing the $\lambda_i^*$-values in the defined way gives optimal solutions for all $\lambda \in [\lambda_i, \lambda_{i+1}]$ which means that we consider all orders of the items that can lead to optimal solutions when inserting the items greedily according to any fixed weightening of expected weights and variances.

In the following, we analyze  the time until an optimal solution with exactly $k$ elements has been produced for 
$$ f_{\lambda_i^*}(x) = \lambda_i^* \mu(x) + (1-\lambda_i^*) v(x) $$
for any $k$, $0 \leq k \leq n$. Note that these $\lambda_i^*$ values allow to obtain all optimal solutions for the set of functions $f_{\lambda}$, $\lambda \in [0,1]$.

For a given $\lambda_i^*$, let the items be ordered such that $f_{\lambda_i^*}(e_1) \leq \ldots \leq  f_{\lambda_i^*}(e_k) \leq \ldots \leq f_{\lambda_i^*}(e_n)$ holds.
An optimal solution for $k$ elements and $\lambda_i^*$ consists of $k$ elements with the smallest $f_{\lambda_i^*}(e_i)$ value. If there is more then one element with the value $f_{\lambda_i^*}(e_k)$ then reordering these elements does not change the objective vector or $f_{\lambda_i^*}$-value.

Assume that optimal solution with $k$ elements for $f_{\lambda_i^*}$ has already been included in the population. Note that for $k=0$ the search point $0^n$ is optimal for any $\lambda \in [0,1]$ and we assume that this search point has already been included in the population.
Picking an optimal solution with $k$ elements for $f_{\lambda_i^*}$ and inserting an element with value $f_{\lambda_i^*}(e_{k+1})$ leads to an optimal solution for $f_{\lambda_i^*}$ with $k+1$ elements.
 We call such a step, picking the solution that is optimal for $f_{\lambda_i^*}$ with $k$ elements and inserting an element with value $f_{\lambda_i^*}(e_{k+1})$, a success. 
Let $X_i$ be the indicator variable for a success in the $i$th step.

We have
\[
\prob{X_i=1} \geq \frac{1}{\pmax en}
\]
as long as an optimal solution has not been obtained for all values of $k$, $1 \leq k \leq n$.
We consider a phase of $T=2e\pmax n^2$ steps.
Let $X=\sum_{i=1}^T X_i$ be the number of successes. The expected number of successes within a phase of $T$ steps where the success in each step is at least $1/(\pmax en)$ is at least $2n$.
The probability to have less than $n$ success is at most $e^{-n/2}$ using Chernoff bounds (see for example Theorem 1.10.5 in \cite{DoerrProbabilisticTools}).
The number of different values of $\lambda_i^*$ is upper bounded by the number of pairs of elements and therefore at most $n^2$. The probability to have not obtained all optimal solutions for all $f_{\lambda_i^*}$ is therefore at most $n^2e^{-n/2}$. Hence, all optimal solutions for all $f_{\lambda_i^*}$ are obtained with probability of at least $1-n^2e^{-n/2}= 1- e^{-\Omega(n)}$ within $T$ steps. As each phase of $T$ steps only requires that the search point $0^n$ is already included in the population and each phase of $T$ steps is successful with probability at least $1- e^{-\Omega(n)}$, the expected number of phases of length $T$ is at most $2$ which completes the proof for the runtime.
As mentioned above, as the same Pareto optimal objectives have been obtained as in the proof of Theorem~\ref{thm:Pareto}, optimal solutions for the problems stated in Equation~\ref{chance-problem} and \ref{chance-problem2} have been obtained.
\end{proof}

\section{Experimental Investigations}
\label{sec:experiments}

We now carry out experimental investigations to see when the $3$-objective formulation is preferable over the bi-objective formulation given in \cite{DBLP:conf/ijcai/0001W22} in practice. As done in \cite{DBLP:conf/ijcai/0001W22}, we investigate a chance constrained version of the minimum dominating set problem where the cost of each node is chosen independently of the others according to a given Normal distribution. Our goal is to provide complementary insights to the theoretical analysis carried out in the previous sections and show when the 3-dimensional approach whose analysis is based on $1$-bit flips is superior to the bi-objective one. Furthermore, we provide insights into the population sizes obtained during the runs of the algorithms. This is a crucial aspect as a larger population size occurred in the $3$-objective approach can potentially make the approach less effective.

\begin{table*}[t]
\small
    \centering
    \begin{tabular}{|c|c||c|c|c|c||c|c|c|c|} \hline 
       \multirow{2}{*}{Graph} & \multirow{2}{*}{weight gype}  &  \multicolumn{2}{|c|}{\bfseries SEMO2D} & \multicolumn{2}{|c||}{\bfseries SEMO3D} &  \multicolumn{2}{|c|}{\bfseries GSEMO2D} & \multicolumn{2}{|c|}{\bfseries GSEMO3D}   \\ 
 &  &    Mean & Std & Mean & Std &     Mean & Std & Mean & Std \\ \hline
cfat200-1 & uniform & 57 & 19 & 2921 & 964& 56 & 16 & 2923 & 929\\
cfat200-2  & uniform &  29 & 11 & 348 & 128& 23 & 10 & 361 & 128\\
ca-netscience & uniform & 69 & 22 & 5531 & 997& 40 & 11 & 4631 & 678\\
ca-GrQc & uniform & 4 & 3 & 7519 & 488& 3 & 1 & 3921 & 262\\
Erdos992 & uniform & 2 & 1 & 4476 & 338& 1 & 1 & 2173 & 153\\ \hline
 
cfat200-1 & uniform-fixed &  1 & 0 & 66 & 12& 1 & 0 & 67 & 11\\
cfat200-2  & uniform-fixed & 1 & 0 & 18 & 4& 1 & 0 & 18 & 4\\
ca-netscience & uniform-fixed & 1 & 0 & 401 & 43& 1 & 0 & 403 & 40\\
ca-GrQc & uniform-fixed & 1 & 0 & 3565 & 251& 1 & 0 & 1942 & 133\\
Erdos992 & uniform-fixed & 1 & 0 & 2217 & 102& 1 & 0 & 1313 & 61\\ \hline

cfat200-1 & degree & 2 & 1 & 335 & 36& 2 & 1 & 340 & 35\\
cfat200-2  & degree & 1 & 1 & 42 & 6& 1 & 0 & 42 & 6\\
ca-netscience & degree & 22 & 8 & 3293 & 764& 18 & 7 & 2981 & 585\\
ca-GrQc & degree & 3 & 2 & 6112 & 371& 3 & 2 & 3240 & 252\\
Erdos992 & degree  & 2 & 1 & 3128 & 166& 2 & 1 & 1725 & 84\\ \hline
    \end{tabular}
    \caption{
    Maximum population size for stochastic minimum weight dominating set. 
    }
    \label{tab:popsize}
\end{table*}

We consider the following problem. Given graph $G=(V,E)$ with $n=|V|$ nodes and weights on the nodes, the aim is to obtain a set of nodes $D \subseteq V$ of minimal weight such that each node of the graph is dominated by $D$, i.e.\  either contained in $D$ or adjacent to a node in $D$. Here the weight $w_i$ of each node $v_i$ is chosen independently of the others according to a Normal distribution $N(\mu_i, \sigma_i^2)$.
Let $c(x)$ be the number of nodes dominated by the given search point $x$. Note that a solution $x$ is feasible iff $c(x)=n$ holds. For the bi-objective formulation, we work with the constraint $c(x) = n$, i.e. each feasible solution has to be a dominating set. For the $3$-objective formulation $c(x)$ counts the number of nodes dominated by $x$ and we pick the best feasible solution, i.e. a solution $x$ with $c(x)= n$, when the algorithm terminates.

\begin{sloppypar}
As done in \cite{DBLP:conf/ijcai/0001W22}, we use the medium size graphs cfat200-1, cfat200-2, ca-netscience from the network repository~\cite{nr} for our experiments.  cfat200-1, cfat200-2 have 200 nodes each whereas ca-netscience has 379 nodes. We also investigate  larger graphs, namely ca-GrQC and Erdoes992 which have 4158 and  6100 nodes, respectively, in order to show the limitations of the $3$-objective approach.
\end{sloppypar}

We consider the following categories for choosing the weights as done in \cite{DBLP:conf/ijcai/0001W22}.
In the \emph{uniform} setting each weight $\mu(u)$ is an integer chosen independently and uniformly at random in $\{n, \ldots, 2 n\}$. The variance $v(u)$ is an integer chosen independently and uniformly at random in $\{n^2, \ldots, 2n^2\}$.
In the \emph{degree-based} setting, we have $\mu(u)= (n + \deg(u))^5/n^4$ where $\deg(u)$ is the degree of node $u$ in the given graph. The variance $v(u)$ is an integer chosen independently and uniformly at random in $\{n^2, \ldots, 2n^2\}$. 
Furthermore, we consider the \emph{uniform-fixed} setting where the expected weights are chosen as in the \emph{uniform} setting, but the variances are set to $2n^2$ for each given node. Our goal here is to study how fixed the variance for each node that therefore making it determined by the number of chosen nodes influences the results compared to the uniform setting.
As done in \cite{DBLP:conf/ijcai/0001W22}, we consider for each combination of graph and weight setting values of $\alpha=1-\beta$ where $\beta \in \{0.2, 0.1, 10^{-2}, 10^{-4}, 10^{-6}, 10^{-8}, 10^{-10}, 10^{-12}, 10^{-14}, 10^{-16}\}$. 

\begin{table*}[htbp]
\tiny
    \centering
    \begin{tabular}{|c|c|c||c|c|c|c|c||c|c|c|c|c|} \hline 
       \multirow{2}{*}{Graph} & \multirow{2}{*}{weight gype}  & \multirow{2}{*}{$\beta$} &  \multicolumn{2}{|c|}{\bfseries SEMO2D} & \multicolumn{2}{|c|}{\bfseries SEMO3D} & &  \multicolumn{2}{|c|}{\bfseries GSEMO2D} & \multicolumn{2}{|c|}{\bfseries GSEMO3D} &  \\ 
 &  & &    Mean & Std & Mean & Std & $p$-value &     Mean & Std & Mean & Std & $p$-value\\ \hline

\multirow{10}{*}{cfat200-1} & \multirow{10}{*}{uniform}& 0.2 & 3618 & 76 & \cellcolor{gray!20}\textbf{3599} & 82 & 0.308& 3615 & 91 & \cellcolor{gray!20}\textbf{3599} & 79 & 0.544\\
& & 0.1 & 3994 & 82 &   \cellcolor{gray!20}\textbf{3970} & 82 & 0.268& 3989 & 96 & \textbf{3972} & 80 & 0.544\\
& & 0.01 & 4877 & 101 & \cellcolor{gray!20}\textbf{4842} & 86 & 0.169& 4866 & 109 & \textbf{4845} & 86 & 0.535\\
& & 1.0E-4 & 6030 & 123 & \cellcolor{gray!20}\textbf{5985} & 94 & 0.128& 6015 & 126 & \textbf{5991} & 98 & 0.455\\
& & 1.0E-6 & 6870 & 139 & \cellcolor{gray!20}\textbf{6824} & 103 & 0.201& 6855 & 138 & \textbf{6832} & 108 & 0.605\\
& & 1.0E-8 & 7562 & 152 & \cellcolor{gray!20}\textbf{7519} & 113 & 0.326& 7546 & 147 & \textbf{7525} & 118 & 0.641\\
& & 1.0E-10 & 8163 & 163 & \cellcolor{gray!20}\textbf{8122} & 123 & 0.469& 8145 & 154 & \textbf{8125} & 125 & 0.751\\
& & 1.0E-12 & 8700 & 172 & \cellcolor{gray!20}\textbf{8660} & 130 & 0.535& 8680 & 159 & \cellcolor{gray!20}\textbf{8660} & 130 & 0.859\\
& & 1.0E-14 & 9190 & 180 & \textbf{9150} & 136 & 0.657& 9169 & 164 & \cellcolor{gray!20}\textbf{9148} & 133 & 0.842\\
& & 1.0E-16 & 9633 & 188 & \textbf{9593} & 142 & 0.636& 9611 & 168 & \cellcolor{gray!20}\textbf{9589} & 137 & 0.865\\ \hline

\multirow{10}{*}{cfat200-2} & \multirow{10}{*}{uniform} & 0.2 & 1797 & 72 & \textbf{1788} & 53 & 0.923& 1791 & 49 & \cellcolor{gray!20}\textbf{1767} & 32 & 0.049\\
& & 0.1 & 2049 & 78 & \textbf{2035} & 55 & 0.865& 2040 & 54 & \cellcolor{gray!20}\textbf{2016} & 37 & 0.074\\
& & 0.01 & 2634 & 92 & \textbf{2617} & 69 & 0.739& 2621 & 72 & \cellcolor{gray!20}\textbf{2593} & 51 & 0.162\\
& & 1.0E-4 & 3394 & 111 & \textbf{3369} & 85 & 0.535& 3381 & 97 & \cellcolor{gray!20}\textbf{3336} & 65 & 0.070\\
& & 1.0E-6 & 3948 & 125 & \textbf{3918} & 95 & 0.511& 3937 & 113 & \cellcolor{gray!20}\textbf{3880} & 71 & 0.044\\
& & 1.0E-8 & 4403 & 134 & \textbf{4372} & 106 & 0.496& 4394 & 124 & \cellcolor{gray!20}\textbf{4329} & 77 & 0.032\\
& & 1.0E-10 & 4799 & 143 & \textbf{4768} & 117 & 0.549& 4793 & 132 & \cellcolor{gray!20}\textbf{4720} & 82 & 0.028\\
& & 1.0E-12 & 5153 & 150 & \textbf{5123} & 125 & 0.559& 5149 & 139 & \cellcolor{gray!20}\textbf{5071} & 85 & 0.024\\
& & 1.0E-14 & 5476 & 157 & \textbf{5447} & 133 & 0.589& 5475 & 145 & \cellcolor{gray!20}\textbf{5391} & 88 & 0.020\\
& & 1.0E-16 & 5769 & 164 & \textbf{5740} & 141 & 0.559& 5769 & 150 & \cellcolor{gray!20}\textbf{5681} & 91 & 0.021\\ \hline

\multirow{10}{*}{ca-netscience} & \multirow{10}{*}{uniform} & 0.2 & 32922 & 1308 & \cellcolor{gray!20}\textbf{32608} & 904 & 0.506& 33042 & 1289 & \textbf{33007} & 1023 & 0.712\\
& & 0.1 & 34456 & 1323 & \cellcolor{gray!20}\textbf{34115} & 907 & 0.544& 34568 & 1302 & \textbf{34514} & 1028 & 0.745\\
& & 0.01 & 38097 & 1361 & \cellcolor{gray!20}\textbf{37694} & 919 & 0.408& 38189 & 1334 & \textbf{38089} & 1040 & 0.848\\
& & 1.0E-4 & 42938 & 1414 & \cellcolor{gray!20}\textbf{42461} & 938 & 0.274& 43012 & 1380 & \textbf{42846} & 1054 & 1.000\\
& & 1.0E-6 & 46527 & 1457 & \cellcolor{gray!20}\textbf{45995} & 951 & 0.255& 46591 & 1415 & \textbf{46377} & 1065 & 0.824\\
& & 1.0E-8 & 49500 & 1493 & \cellcolor{gray!20}\textbf{48923} & 960 & 0.198& 49557 & 1442 & \textbf{49303} & 1076 & 0.712\\
& & 1.0E-10 & 52091 & 1526 & \cellcolor{gray!20}\textbf{51478} & 970 & 0.165& 52145 & 1465 & \textbf{51857} & 1087 & 0.615\\
& & 1.0E-12 & 54416 & 1554 & \cellcolor{gray!20}\textbf{53773} & 979 & 0.147& 54467 & 1487 & \textbf{54150} & 1096 & 0.564\\
& & 1.0E-14 & 56542 & 1581 & \cellcolor{gray!20}\textbf{55873} & 987 & 0.132& 56592 & 1507 & \textbf{56249} & 1105 & 0.487\\
& & 1.0E-16 & 58469 & 1605 & \cellcolor{gray!20}\textbf{57776} & 996 & 0.117& 58517 & 1526 & \textbf{58151} & 1114 & 0.478 \\ \hline

\multirow{10}{*}{cfat200-1} & \multirow{10}{*}{uniform-fixed}& 0.2 & 3891 & 183 & \textbf{3851} & 129 & 0.530& 3813 & 125 & \cellcolor{gray!20}\textbf{3721} & 59 & 0.006\\
& & 0.1 & 4353 & 195 & \textbf{4306} & 135 & 0.464& 4269 & 133 & \cellcolor{gray!20}\textbf{4169} & 59 & 0.006\\
& & 0.01 & 5450 & 224 & \textbf{5385} & 149 & 0.333& 5352 & 151 & \cellcolor{gray!20}\textbf{5235} & 59 & 0.006\\
& & 1.0E-4 & 6913 & 264 & \textbf{6823} & 169 & 0.258& 6795 & 177 & \cellcolor{gray!20}\textbf{6655} & 59 & 0.006\\
& & 1.0E-6 & 7999 & 293 & \textbf{7890} & 183 & 0.234& 7868 & 196 & \cellcolor{gray!20}\textbf{7710} & 59 & 0.006\\
& & 1.0E-8 & 8901 & 317 & \textbf{8776} & 195 & 0.223& 8757 & 212 & \cellcolor{gray!20}\textbf{8586} & 59 & 0.006\\
& & 1.0E-10 & 9688 & 339 & \textbf{9549} & 206 & 0.217& 9534 & 226 & \cellcolor{gray!20}\textbf{9350} & 59 & 0.006\\
& & 1.0E-12 & 10395 & 358 & \textbf{10243} & 216 & 0.206& 10232 & 239 & \cellcolor{gray!20}\textbf{10036} & 59 & 0.006\\
& & 1.0E-14 & 11043 & 376 & \textbf{10878} & 225 & 0.191& 10871 & 250 & \cellcolor{gray!20}\textbf{10665} & 59 & 0.006\\
& & 1.0E-16 & 11630 & 392 & \textbf{11455} & 234 & 0.186& 11450 & 261 & \cellcolor{gray!20}\textbf{11235} & 59 & 0.006 \\ \hline

\multirow{10}{*}{cfat200-2} & \multirow{10}{*}{random-fixed} & 0.2 & 1989 & 116 & \textbf{1980} & 112 & 0.690& 1937 & 104 & \cellcolor{gray!20}\textbf{1866} & 50 & 0.011\\
& & 0.1 & 2307 & 128 & \textbf{2297} & 123 & 0.679& 2249 & 115 & \cellcolor{gray!20}\textbf{2171} & 54 & 0.011\\
& & 0.01 & 3064 & 157 & \textbf{3048} & 149 & 0.554& 2990 & 141 & \cellcolor{gray!20}\textbf{2897} & 63 & 0.011\\
& & 1.0E-4 & 4073 & 196 & \textbf{4049} & 185 & 0.554& 3978 & 176 & \cellcolor{gray!20}\textbf{3864} & 76 & 0.011\\
& & 1.0E-6 & 4822 & 225 & \textbf{4792} & 213 & 0.554& 4712 & 203 & \cellcolor{gray!20}\textbf{4583} & 86 & 0.011\\
& & 1.0E-8 & 5444 & 249 & \textbf{5410} & 236 & 0.554& 5321 & 225 & \cellcolor{gray!20}\textbf{5179} & 94 & 0.011\\
& & 1.0E-10 & 5986 & 269 & \textbf{5948} & 257 & 0.554& 5853 & 244 & \cellcolor{gray!20}\textbf{5700} & 102 & 0.011\\
& & 1.0E-12 & 6474 & 288 & \textbf{6432} & 275 & 0.554& 6330 & 261 & \cellcolor{gray!20}\textbf{6168} & 108 & 0.011\\
& & 1.0E-14 & 6920 & 305 & \textbf{6875} & 292 & 0.554& 6768 & 277 & \cellcolor{gray!20}\textbf{6596} & 114 & 0.011\\
& & 1.0E-16 & 7325 & 321 & \textbf{7276} & 308 & 0.554& 7164 & 291 & \cellcolor{gray!20}\textbf{6984} & 120 & 0.011 \\ \hline

\multirow{10}{*}{ca-netscience} & \multirow{10}{*}{uniform-fixed} & 0.2 & 35378 & 1891 & \textbf{32956} & 844 & 0.000& 34936 & 1747 & \cellcolor{gray!20}\textbf{32926} & 816 & 0.000\\
& & 0.1 & 37239 & 1934 & \textbf{34718} & 844 & 0.000& 36779 & 1785 & \cellcolor{gray!20}\textbf{34687} & 819 & 0.000\\
& & 0.01 & 41659 & 2038 & \textbf{38901} & 844 & 0.000& 41156 & 1878 & \cellcolor{gray!20}\textbf{38869} & 825 & 0.000\\
& & 1.0E-4 & 47551 & 2178 & \textbf{44475} & 844 & 0.000& 46991 & 2004 & \cellcolor{gray!20}\textbf{44439} & 831 & 0.000\\
& & 1.0E-6 & 51927 & 2282 & \textbf{48614} & 843 & 0.000& 51325 & 2098 & \cellcolor{gray!20}\textbf{48576} & 835 & 0.000\\
& & 1.0E-8 & 55559 & 2369 & \textbf{52049} & 842 & 0.000& 54922 & 2177 & \cellcolor{gray!20}\textbf{52009} & 838 & 0.000\\
& & 1.0E-10 & 58729 & 2445 & \textbf{55048} & 842 & 0.000& 58061 & 2246 & \cellcolor{gray!20}\textbf{55006} & 841 & 0.000\\
& & 1.0E-12 & 61577 & 2513 & \textbf{57741} & 843 & 0.000& 60881 & 2309 & \cellcolor{gray!20}\textbf{57698} & 844 & 0.000\\
& & 1.0E-14 & 64184 & 2576 & \textbf{60207} & 843 & 0.000& 63463 & 2366 & \cellcolor{gray!20}\textbf{60162} & 847 & 0.000\\
& & 1.0E-16 & 66548 & 2633 & \textbf{62443} & 844 & 0.000& 65805 & 2418 & \cellcolor{gray!20}\textbf{62397} & 850 & 0.000 \\ \hline

\multirow{10}{*}{cfat200-1} & \multirow{10}{*}{degree}  & 0.2 & 4495 & 143 & \textbf{4392} & 10 & 0.002& 4444 & 115 & \cellcolor{gray!20}\textbf{4387} & 6 & 0.001\\
& & 0.1 & 4835 & 148 & \textbf{4727} & 14 & 0.002& 4781 & 119 & \cellcolor{gray!20}\textbf{4721} & 9 & 0.003\\
& & 0.01 & 5642 & 158 & \textbf{5523} & 25 & 0.001& 5582 & 129 & \cellcolor{gray!20}\textbf{5512} & 16 & 0.004\\
& & 1.0E-4 & 6718 & 172 & \textbf{6584} & 39 & 0.001& 6650 & 143 & \cellcolor{gray!20}\textbf{6566} & 26 & 0.003\\
& & 1.0E-6 & 7517 & 184 & \textbf{7372} & 50 & 0.001& 7443 & 154 & \cellcolor{gray!20}\textbf{7349} & 34 & 0.003\\
& & 1.0E-8 & 8180 & 193 & \textbf{8025} & 59 & 0.001& 8101 & 163 & \cellcolor{gray!20}\textbf{7999} & 40 & 0.003\\
& & 1.0E-10 & 8758 & 202 & \textbf{8596} & 67 & 0.001& 8675 & 171 & \cellcolor{gray!20}\textbf{8567} & 45 & 0.003\\
& & 1.0E-12 & 9278 & 210 & \textbf{9108} & 74 & 0.001& 9191 & 178 & \cellcolor{gray!20}\textbf{9076} & 50 & 0.003\\
& & 1.0E-14 & 9754 & 217 & \textbf{9578} & 81 & 0.001& 9663 & 185 & \cellcolor{gray!20}\textbf{9542} & 55 & 0.003\\
& & 1.0E-16 & 10185 & 223 & \textbf{10003} & 87 & 0.001& 10091 & 191 & \cellcolor{gray!20}\textbf{9965} & 59 & 0.003 \\ \hline

\multirow{10}{*}{cfat200-2} & \multirow{10}{*}{degree} & 0.2 & 3218 & 227 & \textbf{3029} & 154 & 0.033& 3041 & 172 & \cellcolor{gray!20}\textbf{2963} & 4 & 0.027\\
& & 0.1 & 3448 & 235 & \textbf{3256} & 160 & 0.033& 3267 & 178 & \cellcolor{gray!20}\textbf{3185} & 6 & 0.027\\
& & 0.01 & 3996 & 255 & \textbf{3795} & 173 & 0.033& 3803 & 194 & \cellcolor{gray!20}\textbf{3713} & 11 & 0.027\\
& & 1.0E-4 & 4726 & 280 & \textbf{4514} & 193 & 0.033& 4518 & 216 & \cellcolor{gray!20}\textbf{4416} & 17 & 0.027\\
& & 1.0E-6 & 5268 & 300 & \textbf{5048} & 209 & 0.033& 5049 & 232 & \cellcolor{gray!20}\textbf{4938} & 22 & 0.027\\
& & 1.0E-8 & 5718 & 316 & \textbf{5491} & 223 & 0.033& 5490 & 245 & \cellcolor{gray!20}\textbf{5371} & 26 & 0.027\\
& & 1.0E-10 & 6110 & 329 & \textbf{5878} & 235 & 0.033& 5875 & 257 & \cellcolor{gray!20}\textbf{5749} & 30 & 0.027\\
& & 1.0E-12 & 6463 & 342 & \textbf{6225} & 247 & 0.033& 6220 & 267 & \cellcolor{gray!20}\textbf{6089} & 33 & 0.027\\
& & 1.0E-14 & 6786 & 354 & \textbf{6543} & 257 & 0.033& 6537 & 277 & \cellcolor{gray!20}\textbf{6400} & 36 & 0.027\\
& & 1.0E-16 & 7079 & 364 & \textbf{6832} & 267 & 0.033& 6823 & 286 & \cellcolor{gray!20}\textbf{6682} & 38 & 0.027 \\ \hline

\multirow{10}{*}{ca-netscience} & \multirow{10}{*}{degree}  & 0.2 & 28587 & 1535 & \cellcolor{gray!20}\textbf{26148} & 201 & 0.000& 28164 & 1002 & \textbf{26169} & 196 & 0.000\\
& & 0.1 & 30122 & 1580 & \cellcolor{gray!20}\textbf{27636} & 207 & 0.000& 29689 & 1029 & \textbf{27657} & 200 & 0.000\\
& & 0.01 & 33758 & 1686 & \cellcolor{gray!20}\textbf{31158} & 228 & 0.000& 33300 & 1098 & \textbf{31183} & 216 & 0.000\\
& & 1.0E-4 & 38593 & 1828 & \cellcolor{gray!20}\textbf{35840} & 269 & 0.000& 38103 & 1192 & \textbf{35874} & 251 & 0.000\\
& & 1.0E-6 & 42180 & 1936 & \cellcolor{gray!20}\textbf{39313} & 306 & 0.000& 41665 & 1265 & \textbf{39355} & 285 & 0.000\\
& & 1.0E-8 & 45155 & 2026 & \cellcolor{gray!20}\textbf{42192} & 338 & 0.000& 44620 & 1327 & \textbf{42243} & 317 & 0.000\\
& & 1.0E-10 & 47751 & 2104 & \cellcolor{gray!20}\textbf{44705} & 367 & 0.000& 47198 & 1381 & \textbf{44763} & 347 & 0.000\\
& & 1.0E-12 & 50082 & 2175 & \cellcolor{gray!20}\textbf{46962} & 394 & 0.000& 49514 & 1429 & \textbf{47026} & 374 & 0.000\\
& & 1.0E-14 & 52216 & 2239 & \cellcolor{gray!20}\textbf{49027} & 420 & 0.000& 51633 & 1474 & \textbf{49098} & 400 & 0.000\\
& & 1.0E-16 & 54151 & 2297 & \cellcolor{gray!20}\textbf{50900} & 444 & 0.000& 53555 & 1515 & \textbf{50977} & 423 & 0.000\\ \hline

    \end{tabular}
    \caption{
    Results for stochastic minimum weight dominating set with different confidence levels of $\alpha$ where $\alpha=1-\beta$. 
    }
    \label{tab:results}
\end{table*}

The 3-objective formulation is expected to produce much more trade-offs than the bi-objective formulation. We compare SEMO2D and SEMO3D, and GSEMO2D and GSEMO3D, respectively, in terms of the results that they obtain.
Comparing SEMO2D and SEMO3D allows to judge whether the additional objective that increases the population size is helpful in practice even if the mutation operation is highly restrictive. Each algorithm is run for each setting 30 times whereas each run consists of 10M iterations. The results for the $30$ runs carried out for each algorithm are obtained on a the same set of $30$ instances that are generated in the way described above. Note that one run produces results for each considered value of $\alpha$ as we chose from the final population the feasible solution with the smallest weight according to Equation~\ref{eq:weight}. Obviously, finding these solutions can be done in time linear in the size of the final population produced by the considered algorithm.

The results for the maximum population sizes of the considered settings are shown in Table~\ref{tab:popsize}.
For each setting and algorithm, we show the average maximum population size within the $30$ runs and their standard deviations. 
It can be observed that the population sizes encountered by the approaches using the $3$-objective formulations are significantly higher than the maximum population sizes for the bi-objective formulations. Even for relatively small graphs such as cfat200-1, the maximum population sizes for SEMO3D and GSEMO3D are close to $3000$, and reach up to 7500 for SEMO3D and ca-GrQc in the uniform setting. Comparing SEMO3D and GSEMO3D, it is interesting to see that the maximum population size encountered by GSEMO3D is in most cases smaller than for SEMO3D. A possible explanations is that the standard bit mutations used in GSEMO3D are often able to create objective vectors that dominate part of the current population, which reduces the population size.

The optimization results for the different graphs and chance constrained settings are shown in Table~\ref{tab:results}. 
For each setting, we show the average weight value according to Equation~\ref{eq:weight} and standard deviation for the algorithms. 
The best average value in the direct comparison of SEMO2D and SEMO3D, and GSEMO2D and GSEMO3D, respecively, are highlighted in bold. Furthermore, we highlight in grey the best result among all $4$ algorithms.
We also display the $p$-value obtained by the Mann-Whitney test for the comparison of SEMO2D and SEMO3D, and GSEMO2D and GSEMO3D, respectively. We call a result statistically significant if the $p$-value is at most $0.05$.

Our results show that for the graphs cfat200-1, cfat200-2, and ca-netscience, there is usually a strong benefit of using the $3$-objective formulation instead of the bi-objective one.
For the uniform setting where the expected weights and variances are chosen uniformly at random, we can see that either SEMO3D or GSEMO3D obtain the best average weight value for each setting. Interestingly SEMO3D obtains better results than GSEMO3D and SEMO2D better results than GSEMO2D for the ca-netscience graph which might be due to the larger set of trade-offs occurring in the uniform setting, enabling success by carrying out local mutations only.
Considering the uniform-fixed setting, we see that GSEMO3D obtains the best results and significantly outperforms GSEMO2D in all settings. Comparing SEMO3D and SEMO2D for the uniform-fixed setting, we observe smaller average weight values for SEMO3D although the results are only statistically significant for the instances of the graph ca-netscience.

The degree-based setting shows clear advantages for the 3-objec\-tive setting with SEMO3D outperforming SEMO2D and GSEMO3D outperforming GSEMO2D. All results are statistically significant. Interestingly, GSEMO3D has the smallest average cost values for cfat200-1 and cfat200-2, and SEMO3D the smallest average cost values for ca-netscience. We have already seen before that SEMO3D has smaller average costs than GSEMO3D for ca-netscience in the uniform setting which suggest that using only $1$-bit flips instead of standard bit mutations is preferable when working with the given evaluation budget of 10M for this graph of around 400 nodes and variable variances chosen uniformly at random in $\{n^2, \ldots, 2n^2\}$.

For the graphs ca-GrQc and Erdos992 which consist of 4158 and  6100 nodes, respectively, the $3$-objective models often do not obtain a feasible solution within 10M iterations when starting with a solution uniformly at random. A reason for this is the large number of trade-offs resulting in large population sizes before obtaining a feasible solution for the first time. In contrast to this the bi-objective formulation has a clear advantage here as it always works with a population of size $1$ and can only increase its population size when producing feasible solutions. Therefore, we do not display results for these graphs.

\section{Conclusions}
We explored the use of $3$-objective formulation for chance constrained optimization problems using evolutionary multi-objective algorithms.
Our three objective formulation takes the expected weight and its variance each as one objective and adds an additional objective which can be the deterministic objective function or deterministic constraint. For the case of independent Normally distributed uncertainties, we have show that this approach computes optimal solution for the case of the objective function counting the number of chosen elements. Furthermore, we have pointed out that single $1$-bit flips are the key element for the success of our $3$-objective formulation. We showed this through an improved upper bound for the considered problem that showed that the crucial Pareto optimal objective vectors can be obtained through $1$-bit flips.
Our experimental investigations for the chance constrained dominating set problem show the clear advantage of our $3$-objective setting for graphs of moderate size in various stochastic settings.

\section*{Acknowledgments}
This work has been supported by the Australian Research Council (ARC) through grant FT200100536.

\end{document}